\documentclass[runningheads]{llncs}

\usepackage{latexsym}
\usepackage{graphicx}
\usepackage{xcolor}
\usepackage[utf8]{inputenc}
\usepackage{amsmath, amssymb}
\usepackage{pifont}
\usepackage{todonotes}
\usepackage{tikz}
\usetikzlibrary{fit}
\usetikzlibrary{shapes, arrows, positioning, calc, decorations.pathreplacing}
\usetikzlibrary{shapes, arrows.meta, positioning, calc, fit}
\usepackage[pdfborder=0]{hyperref}

\newcommand{\AF}{\mathcal{F}}
\newcommand{\adm}{\mathit{adm}}
\newcommand{\pref}{\mathit{pref}}
\newcommand{\grd}{\mathit{grd}}
\newcommand{\stb}{\mathit{stb}}
\newcommand{\Att}{\mathit{Att}}
\newcommand{\Cred}{\text{Cred}}
\newcommand{\Skep}{\text{Skep}}
\newcommand{\Uni}{\text{Uni}}
\newcommand{\NE}{\text{NE}}
\newcommand{\Ex}{\text{Exists}}
\newcommand{\UniPref}{\Uni_{\pref}}
\definecolor{coloranto}{rgb}{0.85,0.5,0.1}

\title{Complexity of Grounded Semantics and Preferred Semantics in Finitary Argumentation Frameworks}
\titlerunning{Complexity of Grounded and Preferred Semantics in Finitary AFs}

\author{Xu Jinfan \and Luo Jieting}
\authorrunning{Xu Jinfan and Luo Jieting}

\institute{School of Philosophy, Zhejiang University}

\begin{document}
\maketitle

\begin{abstract}
Abstract argumentation frameworks (AFs) introduced by Dung provide a formal foundation for non-monotonic reasoning in artificial intelligence. While decision problems for general infinite AFs typically reside at high levels of the analytical hierarchy ($\Sigma_1^1$ or $\Pi_1^1$), restricting the framework to be computably finitary reduces some of the complexity to the arithmetical hierarchy. In this paper, we present a complexity mapping of grounded and preferred semantics in computably finitary AFs across standard decision problems: credulous acceptance ($\Cred$), skeptical acceptance ($\Skep$), extension existence ($\Ex$), uniqueness ($\Uni$), and non-empty existence ($\NE$). For grounded semantics, credulous and skeptical acceptance are already known to be $\Sigma_1^0$-complete. We show that non-empty existence is also $\Sigma_1^0$-complete, whereas existence and uniqueness are trivial. These classifications are understood within the domain of valid computably finitary representations. For preferred semantics, using a computably finitely branching computation tree, $\Cred_{\pref}$ is shown to be in  $\Pi_1^0$-c and $\NE_{\pref}$ is $\Sigma_2^0$-c. However, it is insufficient to reduce universal quantification and global uniqueness, leaving $\Skep_{\pref}$ in $\Pi_1^1$ and $\UniPref$ in $\Sigma_2^1$-c. Our results show the precise boundary where finitarity succeeds to bring reasoning down to the arithmetical hierarchy and where second-order quantification forces problems back into the analytical hierarchy.

\keywords{Abstract Argumentation \and Finitary AFs \and Computability Theory}
\end{abstract}

\section{Introduction}

Abstract Argumentation Frameworks (AFs), introduced by Dung~\cite{dung1995}, have become a fundamental part of formal reasoning in artificial intelligence. While the computational properties of finite AFs are well-understood, recent years have witnessed a surge of interest in infinite argumentation frameworks (IAF). Infinite AFs enable the formal modeling of domains involving infinite events, continuous temporal reasoning, or infinite debate games; however, they introduce substantial computational challenges. As shown by Andrews and San Mauro ~\cite{andrews2024infinite}, standard reasoning tasks in general infinite AFs often reside at the upper levels of the analytical hierarchy (e.g., being $\Pi_1^1$- or $\Sigma_1^1$-complete), rendering them algorithmically intractable in the general case.

Nevertheless, many realistic infinite scenarios satisfy a natural structural constraint known as the \textit{finitary} restriction: although the total pool of arguments may be infinite, each individual argument is attacked by at most finitely many predecessors. This structural property provides a powerful graphical tool. In finitary AFs, the defense space of an argument can be modeled as a finitely branching tree. By exploiting classical logical tools, such as the $\omega$-continuity of Dung's characteristic operator and König's Lemma, the computational complexity of several core semantics drops from the analytical hierarchy down to the arithmetical hierarchy.

Andrews and San Mauro have already established upper and lower bounds for credulous and skeptical acceptance in computably finitary AFs~\cite{andrews2025complexity}. In particular, grounded acceptance is $\Sigma_1^0$-complete within the domain of valid representations. Building on these results, this paper studies grounded and preferred semantics across five decision problems. For grounded semantics, we recall the acceptance classification and prove that non-empty existence ($\NE_{\grd}^{\mathrm{cfin}}$) is $\Sigma_1^0$-complete; existence and uniqueness are trivial. For preferred semantics, we use the correspondence between admissible sets and preferred extensions to obtain $\Pi_1^0$-completeness of credulous acceptance and $\Sigma_2^0$-completeness of non-empty existence. We also study the analytical complexity of skeptical acceptance and uniqueness under preferred semantics.

\section{Computability Theoretic Background}
In this section, we present the necessary foundations of recursion theory and descriptive complexity used to classify decision problems on infinite structures. We assume all elements of our domain can be effectively encoded as natural numbers $\omega$.

\subsection{Numbers, Strings, and Trees}
To formulate our problems as subsets of $\omega$, it is convenient to encode pairs of numbers into single numbers. For this purpose, we fix a computable bijection $p : \omega \times \omega \to \omega$ and adopt the standard notation $\langle x,y\rangle$ for $p(x,y)$. Furthermore, given a finite set $A=\{x_1,x_2,\ldots,x_n\}\subseteq\omega$, its canonical index is defined as
\[
y=2^{x_1}+2^{x_2}+\cdots+2^{x_n}.
\]
We let $D_y$ denote the finite set whose canonical index is $y$. By leveraging canonical indices, we can quantify over all finite subsets of an argumentation framework.

We denote the set of all finite strings of natural numbers by $\omega^{<\omega}$, and the set of all infinite strings by $\omega^{\omega}$. A string $\sigma$ is called a prefix of $\tau$, written as $\sigma\preceq\tau$, if $\tau$ extends $\sigma$. If neither $\sigma\preceq\tau$ nor $\tau\preceq\sigma$ holds, $\sigma$ and $\tau$ are said to be incomparable. The restriction of an infinite string $\pi \in \omega^{\omega}$ to its first $n$ bits is denoted by $\pi \upharpoonright_n$.

A tree is a set $T\subseteq\omega^{<\omega}$ that is closed under prefixes. An infinite string $\pi\in\omega^{\omega}$ is called a path through $T$ if every finite prefix of $\pi$ belongs to $T$ (i.e., $\sigma\in T$ for all $\sigma\prec\pi$, or $\pi \upharpoonright_n \in T$ for all numbers $n$). We denote the set of all paths through $T$ by $[T]$. A tree $T$ is said to have infinite height if it contains strings of arbitrarily large finite lengths.


\subsection{The Arithmetical Hierarchy: Quantifying over Numbers}

The \textit{arithmetical hierarchy} classifies subsets of $\omega^k$ based on the number of alternating first-order quantifiers ($\exists, \forall$) required to define them over a computable predicate.

Let $R \subseteq \omega^{k+n}$ be a recursive (computable) relation. 
\begin{itemize}
    \item A set $S \subseteq \omega^k$ is in the class $\Sigma_n^0$ if there exists a recursive relation $R$ such that:
    \[ S = \{ \vec{x} \mid \exists y_1 \forall y_2 \exists y_3 \dots Q y_n R(\vec{x}, y_1, \dots, y_n) \} \]
    where the quantifier string contains $n$ alternating quantifiers, and $Q$ is $\exists$ if $n$ is odd and $\forall$ if $n$ is even.
    
    \item A set $S \subseteq \omega^k$ is in the class $\Pi_n^0$ if its complement is in $\Sigma_n^0$. Formally:
    \[ S = \{ \vec{x} \mid \forall y_1 \exists y_2 \forall y_3 \dots Q y_n R(\vec{x}, y_1, \dots, y_n) \} \]
    where $Q$ is $\forall$ if $n$ is odd and $\exists$ if $n$ is even.
    
    \item The class $\Delta_n^0$ is defined as the intersection $\Sigma_n^0 \cap \Pi_n^0$, representing properties that are structurally equivalent to both formats at level $n$.
\end{itemize}

\subsection{The Analytical Hierarchy: Quantifying over Sets}

When the definition of a problem requires quantifying over \text{all possible infinite sets} or \text{functions}, the arithmetical hierarchy is no longer sufficient to classify its complexity. In these instances, we move into the \text{analytical hierarchy}. The fundamental distinction between these hierarchies lies in the domain of quantification. While the arithmetical hierarchy is restricted to quantification over individuals (first-order quantification, e.g., $\exists n \in \omega$), the analytical hierarchy allows for quantification over subsets of individuals (second-order quantification, e.g., $\exists S \subseteq \omega$).

\begin{definition}[Analytical Hierarchy]
A set $S \subseteq \omega^k$ belongs to the class $\Sigma_n^1$ or $\Pi_n^1$ if it can be defined by a second-order formula of the form:
\[ S = \{ \vec{x} \mid (Q_1 X_1)(Q_2 X_2) \dots (Q_n X_n) \Phi(\vec{x}, X_1, \dots, X_n) \} \]
where:
\begin{itemize}
    \item $X_1, X_2, \dots, X_n$ are second-order variables ranging over subsets of $\omega$ (or functions from $\omega$ to $\omega$).
    \item $Q_1, Q_2, \dots, Q_n$ are alternating second-order quantifiers ($\exists^1, \forall^1$).
    \item $\Phi$ is an arithmetical formula containing only first-order quantifiers.
\end{itemize}
A set is $\Sigma_n^1$ if $Q_1 = \exists^1$, and it is $\Pi_n^1$ if $Q_1 = \forall^1$.
\end{definition}

\section{Argumentation Theoretic Background}

\subsection{Abstract Argumentation}

An \textit{argumentation framework} (AF) is a pair $\AF = \langle A, R \rangle$, where $A$ is a set of arguments and $R \subseteq A \times A$ is a binary relation representing attacks. For $a, b \in A$, we say $a$ attacks $b$ if $(a, b) \in R$. A set $S \subseteq A$ is \textit{conflict-free} if there are no $a, b \in S$ such that $(a, b) \in R$.  A set $S \subseteq A$ \textit{defends} an argument $a \in A$ if for every $b \in A$ that attacks $a$, there exists some $c \in S$ that attacks $b$.

\begin{definition}[Characteristic Function]
    The \textit{characteristic function} $\Gamma_{\AF} : 2^A \to 2^A$ is defined as $\Gamma_{\AF}(S) = \{ a \in A \mid S \text{ defends } a \}$.
\end{definition}

\subsection{Argumentation Semantics}

In the extension-based approach to argumentation, a semantics $\sigma$ associates with every AF $\AF = \langle A, R \rangle$ a set of extensions $\sigma(\AF) \subseteq 2^A$. We follow the standard definitions introduced by Dung.

\begin{definition}[Admissible-based Semantics]
Let $\AF = \langle A, R \rangle$ be an AF and $S \subseteq A$ be a conflict-free set.
\begin{itemize}
    \item $S$ is \text{admissible} ($S \in \adm(\AF)$) if $S$ defends all its elements, i.e., $S \subseteq \Gamma_{\AF}(S)$.
    \item $S$ is a \text{complete} extension ($S \in \mathit{com}(\AF)$) if it is a fixed point of the characteristic function, i.e., $S = \Gamma_{\AF}(S)$.
    \item $S$ is the \text{grounded} extension ($S = \grd(\AF)$) if it is the least fixed point of $\Gamma_{\AF}$.
    \item $S$ is a \text{preferred} extension ($S \in \pref(\AF)$) if it is a maximal (w.r.t. $\subseteq$) admissible set.

\end{itemize}
\end{definition}

While the existence of preferred and grounded extensions is guaranteed in any AF (including infinite ones) by Zorn's Lemma and the Knaster-Tarski theorem respectively, the complexity of identifying these sets varies significantly depending on the graph's topology.

\subsection{Computably Finitary AFs and Computational Problems}

We focus on \textit{computably finitary} AFs, which are infinite structures where the local interaction is effectively bounded.

\begin{definition}[Finitary AF]
An AF $\AF = \langle A, R \rangle$ is \textit{finitary} if for every $a \in A$, the set of its attackers $\Att(a) = \{ b \in A \mid (b, a) \in R \}$ is finite. 
\end{definition}

\begin{definition}[Computably Finitary AF]
An argumentation framework $\AF = \langle A, R \rangle$, where $A$ is the set of arguments and $R \subseteq A \times A$ is the attack relation, is said to be computably finitary if it satisfies the following two conditions ~\cite{andrews2025complexity}:

\begin{enumerate}
    \item \text{Recursive Set of Arguments}: The set $A$ is a \textit{recursive set}. Formally, a set $A \subseteq \omega$ is recursive (or decidable) if its characteristic function $\chi_A: \omega \to \{0, 1\}$ is a total computable function, defined as:
    \[ 
    \chi_A(n) = \begin{cases} 1 & \text{if } n \in A \\ 0 & \text{if } n \notin A \end{cases}
    \]
    In the context of infinite argumentation, $A$ is typically identified with the set of natural numbers $\omega$.

    \item \text{Computable Attackers}: There exists a total computable function $f: A \to \omega$ such that for every argument $a \in A$, $f(a)$ returns the \textit{canonical index} of the finite set of its attackers, $\Att(a) = \{ b \in A \mid (b, a) \in R \}$. 
\end{enumerate}
\end{definition}

Following the methodology of Andrews and San Mauro ~\cite{andrews2024infinite}, we analyze five fundamental decision problems for each semantics $\sigma$:
\begin{enumerate}
    \item \text{Credulous Acceptance} ($\Cred_{\sigma}$): Given $\AF$ and $x \in A$, is $x$ contained in at least one $S \in \sigma(\AF)$?
    \item \text{Skeptical Acceptance} ($\Skep_{\sigma}$): Given $\AF$ and $x \in A$, is $x$ contained in every $S \in \sigma(\AF)$?
    \item \text{Existence} ($\Ex_{\sigma}$): Does there exist a set $S \in \sigma(\AF)$? (Note: for grounded/preferred, this is trivial).
     \item \text{Non-empty Existence} ($\text{NE}_\sigma$): Does there exist a non-empty set $S \in \sigma(\AF)$? 
    \item \text{Uniqueness} ($\Uni_{\sigma}$): Is there exactly one set $S \in \sigma(\AF)$?
\end{enumerate}
For the computably finitary representation, fix a standard enumeration $(\varphi_e)_{e\in\omega}$ of the partial computable functions and let
\[
\mathrm{Tot}=\{e\in\omega\mid \varphi_e\text{ is total}\}.
\]
For each $e\in\mathrm{Tot}$, let $\mathcal{F}_e$ have argument set $A=\{a_x\mid x\in\omega\}$, with the finite attacker set of $a_x$ encoded by $D_{\varphi_e(x)}$. Here a set of numerical codes is identified with the corresponding set of arguments, so that
\[
\Att_{\mathcal{F}_e}(a_x)=\{a_y\mid y\in D_{\varphi_e(x)}\}.
\]
Thus, a valid index provides an effective finite list of all attackers of each argument. Complexity classifications below are understood within $\mathrm{Tot}$ (or $\mathrm{Tot}\times\omega$ for acceptance): the upper-bound formulas need only agree with the decision problems on valid indices, and hardness reductions must always produce valid indices. They do not include the complexity of testing whether an arbitrary index belongs to $\mathrm{Tot}$.

We code the decision problems as follows:
\begin{enumerate}
    \item $\Cred_\sigma^{\mathrm{cfin}}:=\{\langle e,n\rangle\mid e\in\mathrm{Tot},\ (\exists S\in\sigma(\mathcal{F}_e))(a_n\in S)\}$;
    \item $\Skep_\sigma^{\mathrm{cfin}}:=\{\langle e,n\rangle\mid e\in\mathrm{Tot},\ (\forall S\in\sigma(\mathcal{F}_e))(a_n\in S)\}$;
    \item $\Ex_\sigma^{\mathrm{cfin}}:=\{e\in\mathrm{Tot}\mid \sigma(\mathcal{F}_e)\neq\emptyset\}$;
    \item $\NE_\sigma^{\mathrm{cfin}}:=\{e\in\mathrm{Tot}\mid (\exists S\in\sigma(\mathcal{F}_e))(S\neq\emptyset)\}$;
    \item $\Uni_\sigma^{\mathrm{cfin}}:=\{e\in\mathrm{Tot}\mid |\sigma(\mathcal{F}_e)|=1\}$.
\end{enumerate}
For grounded semantics, $\sigma(\mathcal{F}_e)$ in these definitions is the singleton $\{\grd(\mathcal{F}_e)\}$. We omit the superscript $\mathrm{cfin}$ when the representation is clear from context. We also write $\Cred_\sigma(\mathcal{F}_e)$ and $\Skep_\sigma(\mathcal{F}_e)$ for the corresponding sets of accepted arguments.

\section{Grounded Semantics}
We begin with grounded semantics. The grounded extension $G$ is the least fixed point of $\Gamma_{\AF}(S) = \{ a \in A \mid S \text{ defends } a \}$.

\begin{lemma}[$\omega$-continuity by Dung ~\cite{dung1995}]
If $\AF$ is finitary, $\Gamma_{\AF}$ is $\omega$-continuous, i.e., $\Gamma(\bigcup S_i) = \bigcup \Gamma(S_i)$. Consequently, $G = \bigcup_{n < \omega} \Gamma^n(\emptyset)$.
\end{lemma}

\begin{proposition}\label{prop:gpre}
Let $\AF = \langle A, R \rangle$ be a computably finitary argumentation framework. For any argument $x \in A$ and any finite step $n \in \omega$, the predicate $x \in \Gamma_{\AF}^n(\emptyset)$ is decidable ($\Delta_1^0$).
\end{proposition}

\begin{proof}
We proceed by induction on $n$.
\begin{itemize}
    \item \text{Base Case ($n=0$):} By definition, $\Gamma_{\AF}^0(\emptyset) = \emptyset$. The predicate $x \in \emptyset$ always evaluates to false, which is trivially decidable.
    
    \item \text{Inductive Step:} Assume that for a given $n \in \omega$, the predicate $y \in \Gamma_{\AF}^n(\emptyset)$ is decidable for all $y \in A$. The next iteration of the characteristic operator is defined as:
    \[ \Gamma_{\AF}^{n+1}(\emptyset) = \{ x \in A \mid \forall y \in A ( (y, x) \in R \to \exists z \in \Gamma_{\AF}^n(\emptyset) ( (z, y) \in R) ) \} \]
    Since $\AF$ is computably finitary, the set $\Att(x)$ is finite and its canonical index can be effectively computed. Thus, the universal quantification over $y$ reduces to a finite check over the explicit members of $\Att(x)$. For each individual attacker $y \in \Att(x)$, we must check if $y$ is counter-attacked by some element currently in $\Gamma_{\AF}^n(\emptyset)$. This is equivalent to verifying whether $\Att(y) \cap \Gamma_{\AF}^n(\emptyset) \neq \emptyset$. Again, because $\AF$ is computably finitary, $\Att(y)$ is a finite, computable set. By the inductive hypothesis, membership in $\Gamma_{\AF}^n(\emptyset)$ is decidable. Therefore, testing whether the finite set $\Att(y)$ contains at least one element belonging to $\Gamma_{\AF}^n(\emptyset)$ takes finite time and is decidable. Thus, the entire predicate $x \in \Gamma_{\AF}^{n+1}(\emptyset)$ is decidable ($\Delta_1^0$).
\end{itemize}
By mathematical induction, the predicate $x \in \Gamma_{\AF}^n(\emptyset)$ is decidable for any $n \in \omega$.
\end{proof}

\subsection{Credulous and Skeptical Acceptance}

Andrews and San Mauro have already established upper and lower bounds for credulous and skeptical acceptance in computably finitary AFs. Their construction for complete semantics directly uses the grounded extension to encode computably enumerable sets, yielding $\Sigma_1^0$-completeness of grounded acceptance within $\mathrm{Tot}\times\omega$~\cite{andrews2025complexity}.

Since the grounded extension is unique, credulous and skeptical acceptance coincide. The upper bound can also be seen from $\omega$-continuity and Proposition~\ref{prop:gpre}:
\[
a_x\in\grd(\mathcal{F}_e)
\quad\Longleftrightarrow\quad
\exists n\in\omega\;
\bigl(a_x\in\Gamma_{\mathcal{F}_e}^{n}(\emptyset)\bigr).
\]
On valid indices, the finite-stage predicate is decidable uniformly in $e,x,n$. Equivalently, a finite computation witnessing membership at some stage can be effectively verified. Hence $\Cred_{\grd}^{\mathrm{cfin}}$ and $\Skep_{\grd}^{\mathrm{cfin}}$ are $\Sigma_1^0$-complete within $\mathrm{Tot}\times\omega$.

\subsection{Non-empty Existence}

\begin{proposition}[Non-empty existence under grounded semantics]
\label{prop:grounded-nonempty}
Under the computably finitary representation, $\NE_{\grd}^{\mathrm{cfin}}$ is $\Sigma_1^0$-complete within $\mathrm{Tot}$.
\end{proposition}

\begin{proof}
First observe that, for every argumentation framework $\mathcal{F}$,
\[
\grd(\mathcal{F})\neq\emptyset
\quad\Longleftrightarrow\quad
\Gamma_{\mathcal{F}}(\emptyset)\neq\emptyset
\quad\Longleftrightarrow\quad
\exists a\in A\;\bigl(\Att_{\mathcal{F}}(a)=\emptyset\bigr).
\]
Indeed, $\Gamma_{\mathcal{F}}(\emptyset)$ consists precisely of the unattacked arguments and is contained in the grounded extension. If there are no unattacked arguments, then $\Gamma_{\mathcal{F}}(\emptyset)=\emptyset$, so $\emptyset$ is already a fixed point and $\grd(\mathcal{F})=\emptyset$.

For $e\in\mathrm{Tot}$, the attacker set of $a_x$ is encoded by $D_{\varphi_e(x)}$, and therefore
\[
e\in\NE_{\grd}^{\mathrm{cfin}}
\quad\Longleftrightarrow\quad
\exists x\in\omega\;
\bigl(D_{\varphi_e(x)}=\emptyset\bigr).
\]
Given a valid index $e$ and an argument code $x$, we can compute this finite set and decide whether it is empty. More explicitly, the condition can be expressed as the existence of $x$ and a finite computation in which $\varphi_e(x)$ halts with output $0$, the canonical index of the empty set. This gives a $\Sigma_1^0$ definition on the valid representation domain.

For $\Sigma_1^0$-hardness, we give a many-one reduction from the halting set
\[
K=\{e\mid \varphi_e(e)\text{ eventually halts}\}.
\]
Given $e$, construct a framework $\mathcal{F}_{f(e)}$ with argument set $\{a_n\mid n\in\omega\}$. For each $s\in\omega$, set
\[
\Att(a_{2s})=
\begin{cases}
\emptyset, & \text{if }\varphi_e(e)\text{ halts within the first }s\text{ steps},\\
\{a_{2s+1}\}, & \text{otherwise},
\end{cases}
\]
and set $\Att(a_{2s+1})=\{a_{2s+1}\}$.
Given $e$ and $s$, simulating $\varphi_e(e)$ for the finite number of steps $s$ suffices to compute these attacker sets. Their canonical indices are respectively $0$ or $2^{2s+1}$ for $a_{2s}$, and $2^{2s+1}$ for $a_{2s+1}$. The construction is uniform, so the $s$-$m$-$n$ theorem yields a total computable index map $f$, with $f(e)\in\mathrm{Tot}$ for every $e$.

If $e\in K$ and $\varphi_e(e)$ halts within $t$ steps, then $a_{2s}$ is unattacked for every $s\geq t$, and hence $\grd(\mathcal{F}_{f(e)})\neq\emptyset$. If $e\notin K$, every $a_{2s}$ is attacked by $a_{2s+1}$, and every $a_{2s+1}$ attacks itself. Thus there are no unattacked arguments and $\grd(\mathcal{F}_{f(e)})=\emptyset$. Consequently,
\[
e\in K
\quad\Longleftrightarrow\quad
f(e)\in\NE_{\grd}^{\mathrm{cfin}}.
\]
Therefore $\NE_{\grd}^{\mathrm{cfin}}$ is $\Sigma_1^0$-complete within $\mathrm{Tot}$.
\end{proof}

\subsection{Existence and Uniqueness}

Since every argumentation framework has exactly one grounded extension, both $\Ex_{\grd}^{\mathrm{cfin}}$ and $\Uni_{\grd}^{\mathrm{cfin}}$ are trivial on the valid representation domain: every $e\in\mathrm{Tot}$ is a positive instance.

\section{Preferred Semantics}
We now turn to preferred semantics. A preferred extension is a maximal (w.r.t.\ $\subseteq$) admissible set. 
Because every chain of admissible sets has an upper bound, by Zorn's Lemma, a preferred extension always exists, and any argument contained in an admissible set is also contained in at least one preferred extension.

\begin{lemma} \label{23}
Let $\AF = \langle A, R \rangle$ be an argumentation framework. For any admissible set $S_0 \in \adm(\AF)$, there exists a preferred extension $E \in \pref(\AF)$ such that $S_0 \subseteq E$.
\end{lemma}

\begin{proof}
Let $\mathcal{M} = \{ S \in \adm(\AF) \mid S_0 \subseteq S \}$ be the collection of all admissible sets containing $S_0$. The poset $\langle \mathcal{M}, \subseteq \rangle$ is partially ordered by set inclusion. We show that every totally ordered chain in $\mathcal{M}$ has an upper bound in $\mathcal{M}$.

Let $\mathcal{C} \subseteq \mathcal{M}$ be an arbitrary totally ordered chain, and consider its union $U = \bigcup_{X \in \mathcal{C}} X$. Clearly, $U$ is an upper bound of every $X \in \mathcal{C}$ under set inclusion. To establish that $U \in \mathcal{M}$, we must prove that $U$ is an admissible set (i.e., conflict-free and self-defending):

\begin{itemize}
    \item \textbf{Conflict-freeness of $U$}: Suppose for contradiction that $U$ is not conflict-free. Then, there exist $a, b \in U$ such that $(a, b) \in R$. By definition of the union, there exist $X_a, X_b \in \mathcal{C}$ such that $a \in X_a$ and $b \in X_b$. Since $\mathcal{C}$ is a totally ordered chain, we must have either $X_a \subseteq X_b$ or $X_b \subseteq X_a$. Assume $X_a \subseteq X_b$, which implies $a, b \in X_b$. However, since $X_b \in \mathcal{C} \subseteq \mathcal{M}$, $X_b$ is admissible and must be conflict-free, contradicting $(a, b) \in R$. Thus, $U$ is conflict-free.
    
    \item \textbf{Self-defense of $U$}: Let $a \in U$ be an arbitrary argument, and suppose there exists $y \in A$ such that $(y, a) \in R$. Since $a \in U$, there exists some $X_a \in \mathcal{C}$ such that $a \in X_a$. Since $X_a$ is admissible, it defends $a$, meaning there exists $z \in X_a$ such that $(z, y) \in R$. Since $X_a \subseteq U$, we have $z \in U$. Thus, $U$ defends all its elements.
\end{itemize}

Since $U$ is both conflict-free and self-defending, $U \in \adm(\AF)$. Furthermore, since $S_0 \subseteq X$ for all $X \in \mathcal{C}$, we have $S_0 \subseteq U$, which guarantees $U \in \mathcal{M}$. By Zorn's Lemma, $\langle \mathcal{M}, \subseteq \rangle$ contains at least one maximal element $E$. By definition, $E$ is an admissible set containing $S_0$. The maximality of $E$ in $\mathcal{M}$ ensures that $E \in \pref(\AF)$.
\end{proof}

For proving the upper bound complexity of decision problems $\Skep_{\pref}$ and $\UniPref$, we will utilize the technique from \cite{andrews2025complexity,andrews2025complexity} to build a bridge between paths in a computably finitary tree and admissible sets, as well as König's Lemma.
\begin{definition} 
Any given string $\sigma \in \omega^{<\omega}$ defines two subsets of arguments in $\mathcal{F}$:
\begin{itemize}
    \item $\mathrm{In}_\sigma = D \cup \{a_j : \sigma(2j) = 1\} \cup \{a_k : (\exists j)(\sigma(2j + 1) = k + 1)\} \cup \{a_i : (\exists j)(\sigma(2j + 1) > 0 \land g_j^+ = a_i)\}$;
    \item $\mathrm{Out}_\sigma = E \cup \{a_j : \sigma(2j) = 0\} \cup \{a_i : (\exists j)(\sigma(2j + 1) > i + 1 \land a_i \rightarrowtail g_j^-)\} \cup \{a_i : (\exists j)(\sigma(2j + 1) = 0 \land g_j^+ = a_i)\}$.
\end{itemize}
We define $T_{\mathcal{F}}$ as the set of $\sigma$ such that
\begin{itemize}
    \item $\mathrm{In}_\sigma$ does not contain elements $a_i \rightarrowtail a_j$ with $i, j < |\sigma|$;
    \item $\mathrm{In}_\sigma \cap \mathrm{Out}_\sigma \cap \{a_i : i < |\sigma|\} = \emptyset$;
    \item If $2j < |\sigma|$, then $\sigma(2j) \in \{0, 1\}$;
    \item If $2j + 1 < |\sigma|$ and $\sigma(2j + 1) = k + 1$, then $a_k \rightarrowtail g_j^-$.
\end{itemize}
Note that if $\mathcal{F}$ is computable, then $T_{\mathcal{F}}$ is computable, and if $\mathcal{F}$ is computably finitary, then $T_{\mathcal{F}}$ is computably finitely branching.
\end{definition}

\begin{theorem}[] \label{thm:andrews_tree}
Let $\AF = \langle A, R \rangle$ be a computably finitary argumentation framework and $x \in A$. There exists a non-empty admissible set $S \in \adm(\AF)$ with $x \in S$ if and only if the computation tree $T_{\mathcal{F}}$ contains an infinite path $\pi \in [T_{\mathcal{F}}]$. 
\end{theorem}
\begin{proof}
    See ~\cite{andrews2025complexity} for a complete proof.
\end{proof}

\begin{lemma}[König's Lemma]
Let $T$ be an infinite tree. If every node in $T$ has a finite number of children (i.e., $T$ is finitely branching), then $T$ contains at least one infinite path.
\end{lemma}

\subsection{Credulous Acceptance}

\begin{lemma} \label{3}
Let $\AF = \langle A, R \rangle$ be an argumentation framework. An argument $x \in A$ is credulously accepted under the preferred semantics iff it is credulously accepted under the admissible semantics. That is,
\[ \text{Cred}_{\pref}(\AF) = \text{Cred}_{\adm}(\AF) \]
\end{lemma}

\begin{proof}
We establish the identity by bidirectional inclusion:
\begin{itemize}
    \item ($\subseteq$) Let $x \in \text{Cred}_{\pref}(\AF)$. There exists a preferred extension $E \in \pref(\AF)$ containing $x$. Since every preferred extension is admissible by definition ($E \in \adm(\AF)$), $E$ serves as a witnessing admissible set. Thus, $x \in \text{Cred}_{\adm}(\AF)$.
    
    \item ($\supseteq$) Let $x \in \text{Cred}_{\adm}(\AF)$. There exists an admissible set $S_0 \in \adm(\AF)$ containing $x$. By Lemma \ref{23}, there exists a preferred extension $E \in \pref(\AF)$ such that $S_0 \subseteq E$. Since $x \in S_0$, we have $x \in E$, which implies $x \in \text{Cred}_{\pref}(\AF)$.
\end{itemize}
\end{proof}

\begin{theorem}
In a computably finitary AF $\AF$, $\Cred_{\pref}(\AF)$ is $\Pi_1^0$-complete.
\end{theorem}

\begin{proof}
In Dung's abstract argumentation frameworks, By Lemma \ref{3}, for any argument $x \in A$, $x$ belongs to some preferred extension $E \in \pref(\AF)$ if $x$ belongs to some admissible set $S \in \adm(\AF)$. Thus, we have $\text{Cred}_{\pref}(\AF) = \text{Cred}_{\adm}(\AF)$.

By the established results ~\cite{andrews2024infinite}, while $\text{Cred}_{\adm}$ is $\Sigma_1^1$-complete for general infinite frameworks, introducing the computably finitary restriction reduces its complexity to $\Pi_1^0$. Due to the extensional identity $\text{Cred}_{\pref}(\AF) = \text{Cred}_{\adm}(\AF)$, the $\Pi_1^0$-completeness of $\text{Cred}_{\pref}$ follows immediately.
\end{proof}

\subsection{Non-empty Existence}

\begin{lemma}\label{2}
Let $\AF = \langle A, R \rangle$ be an argumentation framework. There exists a non-empty preferred extension $E \in \pref(\AF)$ iff there exists a non-empty admissible set $S \in \adm(\AF)$.
\end{lemma}

\begin{proof}
We prove the equivalence in both directions:

\text{$(\implies)$} Assume there exists a non-empty preferred extension $E \in \pref(\AF)$. Since every preferred extension is an admissible set, $E \in \adm(\AF)$. Because $E \neq \emptyset$, it immediately serves as a witnessing non-empty admissible set.

\text{$(\impliedby)$} Assume there exists a non-empty admissible set $S_0 \in \adm(\AF)$. By Lemma \ref{23}, there exists a preferred extension $E \in \pref(\AF)$ such that $S_0 \subseteq E$. Since $S_0 \neq \emptyset$, it follows immediately that $E \neq \emptyset$.
\end{proof}

\begin{theorem}
In a computably finitary AF $\AF$, the non-empty existence problem under preferred semantics ($\NE_{\pref}$) is $\Sigma_2^0$-complete.
\end{theorem}

\begin{proof}
By Lemma \ref{2}, an argumentation framework $\AF$ contains a non-empty preferred extension if and only if it contains a non-empty admissible set. This equivalence immediately establishes an equivalence between the decision problems of non-empty existence under preferred and admissible semantics. That is,$\NE_{\pref} = \NE_{\adm}$.

Since the non-empty existence problem under admissible semantics ($\NE_{\adm}$) has already been established as $\Sigma_2^0$-complete in computably finitary AFs, it follows directly that $\NE_{\pref}$ is also $\Sigma_2^0$-c.
\end{proof}

\subsection{Skeptical Acceptance}

So far all problems remain arithmetical; skeptical acceptance is the turning point. Under the preferred semantics, an argument is skeptically accepted if and only if it belongs to every preferred extension (defined as maximal admissible sets under set inclusion).

\begin{lemma}\label{dieout}
A path $P \in [T_{\mathcal{F}}]$ corresponds to a preferred extension iff for every $n \in \omega$, and for every child node $\tau$ of $P\upharpoonright n$ in $T_{\mathcal{F}}$:
\[ \text{If } \tau \neq P\upharpoonright (n+1), \text{ then the subtree rooted at $\tau$ contains no infinite path.} \]
\end{lemma}
\begin{proof}
Let $P, T \in [T_{\mathcal{F}}]$ be two infinite paths corresponding to admissible sets $S$ and $S'$, respectively. The set inclusion $S \subset S'$ corresponds to $P(2n) = 1 \implies T(2n) = 1$, which also corresponds to $P$ and $T$ share a prefix. Thus, the maximality of an admissible set $S$ (i.e., $S$ is a preferred extension) can be characterized by $T_{\mathcal{F}}$ as a path where there is no divergent infinite branch.
\end{proof}

\begin{theorem}
In a computably finitary argumentation framework $\AF$, the skeptical acceptance problem under preferred semantics ($\Skep_{\pref}$) is in $\Pi_1^1$.
\end{theorem}

\begin{proof}
An argument $x \in A$ fails to be skeptically accepted with respect to preferred semantics ($x \notin \Skep_{\pref}(\AF)$) if and only if there exists a preferred extension $E \in \pref(\AF)$ such that $x \notin E$. The admissible sets correspond to the set of infinite paths $[T_{\mathcal{F}}]$ through a computation tree $T^{\AF}$. Using Lemma~\ref{dieout}, the non-skeptical acceptance problem $x \notin \Skep_{\pref}(\AF)$ is equivalent to the assertion:
\begin{quote}
There exists an infinite path $P \in [T_{\mathcal{F}}]$ such that its corresponding admissible set excludes argument $x$, and for any node $\sigma = P\upharpoonright n$, every child $\tau \neq P\upharpoonright (n+1)$ roots a subtree $(T_{\mathcal{F}})_\tau$ that eventually dies out (i.e., contains no infinite paths).
\end{quote}
The condition that a sub-tree contains no infinite paths is a $\Pi_1^0$ property by König's Lemma. However, asserting the existence of an infinite path $P$ with a maximality condition that evaluates infinitely many subtrees introduces an existential quantification over infinite paths ($\exists P \in [T_{\mathcal{F}}]$), placing the non-skeptical acceptance problem in the $\Sigma_1^1$ level of the analytical hierarchy. Taking the complement, the skeptical acceptance problem $x \in \Skep_{\pref}(\AF)$ is expressible as a $\Pi_1^1$ formula.
\end{proof}

\subsection{The Uniqueness Problem}

\begin{lemma}
Let $\AF = \langle A, R \rangle$ be an argumentation framework. $\AF$ has a unique preferred extension iff there exists  $E \in \adm(\AF)$, for all $S \in \adm(\AF)$, $ S \subseteq E$.
\end{lemma}
\begin{proof}
Proof is straightforward by the definition of preferred extensions.
\end{proof}

\begin{theorem}\label{pu}
In a computably finitary argumentation framework $\AF$, the uniqueness problem for preferred extensions ($\UniPref$) is in $\Sigma_2^1$.
\end{theorem}

\begin{proof}
Let $T_{\mathcal{F}}$ be the computable, finitely branching tree whose infinite paths $[T_{\mathcal{F}}]$ correspond to all admissible sets in $\AF$. 

Encoding argument membership(i.e., let $S$ be an admissible set and $P$ be its corresponding path, argument $a_n \in S \iff P(2n) = 1$), a framework $\AF$ belongs to $\UniPref$ if and only if there exists a witness path $P \in [T_{\mathcal{F}}]$ such that every path $T \in [T_{\mathcal{F}}]$ is contained in $P$. Formally:
\[ \AF \in \UniPref \iff \exists P \in [T_{\mathcal{F}}] \; \forall T \in [T_{\mathcal{F}}] \; \forall n \in \omega \; \big( T(2n) = 1 \implies P(2n) = 1 \big) \]

Analyzing the quantifier complexity of this formalization:
\begin{enumerate}
    \item The inner relation $\forall n \in \omega \, \big( T(2n) = 1 \implies P(2n) = 1 \big)$ evaluates a universal quantifier over the natural numbers on a decidable relation, forming a $\Pi_1^0$ predicate.
    \item Universal quantification over all infinite paths $T \in [T_{\mathcal{F}}]$ adds a universal second-order quantifier.
    \item Existential quantification over the witness path $P \in [T_{\mathcal{F}}]$ adds an existential second-order quantifier.
\end{enumerate}

Consequently, the quantifier prefix takes the form $\exists P \, \forall T \, \forall n$, which places the uniqueness problem $\UniPref$ precisely in the $\Sigma_2^1$ tier of the analytical hierarchy.
\end{proof}
\begin{theorem}[$\Sigma_2^1$-Hardness of Uniqueness]\label{hd}
In computably finitary argumentation frameworks, the uniqueness problem for preferred extensions $\UniPref$ is $\Sigma_2^1$-hard.
\end{theorem}
\begin{proof}
We reduce a canonical $\Sigma_2^1$-complete set $\mathcal{E}$ to $\UniPref$. Let $e \in \mathcal{E}$ be represented as:
\[
e \in \mathcal{E} \iff \exists P \in 2^\omega \, \forall T \in 2^\omega \, \exists n \in \omega \; R(e, P \upharpoonright n, T \upharpoonright n)
\]
where $R$ is a computable relation over finite sequences. We construct an infinite argumentation framework $\AF_e = \langle A, \Att \rangle$ corresponding to $e$.

The set of arguments $A$ is defined as:
\[
A = \{p_\sigma, e^P_\sigma \mid \sigma \in 2^{<\omega}\} \cup \{t_\tau, e^T_\tau \mid \tau \in 2^{<\omega}\} \cup \{U\}
\]
where $p_\sigma$ and $t_\tau$ represent finite path choices in the candidate tree $P$ and challenger tree $T$ respectively, $e^P_\sigma$ and $e^T_\tau$ act as enforcer arguments, and $U$ is the global dominator argument.

The attack relation $\Att$ consists of the following pairs for all $\sigma, \tau \in 2^{<\omega}$:
\begin{enumerate}
    \item 
    $p_{\sigma \cdot 0} \leftrightarrow p_{\sigma \cdot 1}$ and $t_{\tau \cdot 0} \leftrightarrow t_{\tau \cdot 1}$.
    
    \item 
    $e^P_\sigma \to p_\sigma$ with children counter-attacking $p_{\sigma \cdot 0} \to e^P_\sigma$ and $p_{\sigma \cdot 1} \to e^P_\sigma$. Similarly, $e^T_\tau \to t_\tau$ with $t_{\tau \cdot 0} \to e^T_\tau$ and $t_{\tau \cdot 1} \to e^T_\tau$.
    
    \item 
    For $|\sigma| = |\tau| = n$, if $R(e, \sigma, \tau)$ holds (is True), then $p_\sigma \to t_\tau$.
    
    \item 
    For all $\tau \in 2^{<\omega}$, $t_\tau \to U$. Conversely, $U$ attacks all non-root tree arguments: $U \to p_\sigma, e^P_\sigma, t_\tau, e^T_\tau$ for all $\sigma \neq \emptyset, \tau \neq \emptyset$.
\end{enumerate}
\begin{figure}[htbp]
\centering
\begin{tikzpicture}[
    >=Stealth,
    node distance=1.2cm and 1.8cm,
    arg/.style={circle, draw=black!80, thick, fill=blue!5, minimum size=9mm, inner sep=0pt, font=\small},
    enf/.style={rectangle, draw=red!80, thick, fill=red!5, minimum size=8mm, inner sep=0pt, font=\small},
    dom/.style={ellipse, draw=purple!80, thick, fill=purple!10, minimum size=9mm, font=\bfseries\small},
    box/.style={draw=black!30, dashed, rounded corners, inner sep=10pt},
    attack/.style={->, thick, color=red!70!black},
    counter/.style={->, thick, color=green!50!black}, 
    mutual/.style={<->, thick, color=blue!70!black},
    eval/.style={->, very thick, dashed, color=orange!80!black}
]

    \node[enf] (e_P) {$e^P_\sigma$};
    \node[arg, below=1.0cm of e_P] (p_sigma) {$p_\sigma$};
    
    \node[arg, below left=1.2cm and 0.8cm of p_sigma] (p_0) {$p_{\sigma\cdot 0}$};
    \node[arg, below right=1.2cm and 0.8cm of p_sigma] (p_1) {$p_{\sigma\cdot 1}$};

    \draw[attack] (e_P) -- node[right, font=\tiny] {$\text{}$} (p_sigma); 
    \draw[mutual] (p_0) -- node[above, font=\tiny] {$\text{}$} (p_1); 

    \draw[counter, bend left=35] (p_0) to node[left, font=\tiny] {$p_{\sigma\cdot 0} \to e^P_\sigma$} (e_P);
    \draw[counter, bend right=35] (p_1) to node[right, font=\tiny] {} (e_P);

    \node[box, fit=(e_P) (p_sigma) (p_0) (p_1), label=above:\textbf{$P$-Tree (Witness)}] (p_tree) {};

    \node[enf, right=5.0cm of e_P] (e_T) {$e^T_\tau$};
    \node[arg, below=1.0cm of e_T] (t_tau) {$t_\tau$};
    
    \node[arg, below left=1.2cm and 0.8cm of t_tau] (t_0) {$t_{\tau\cdot 0}$};
    \node[arg, below right=1.2cm and 0.8cm of t_tau] (t_1) {$t_{\tau\cdot 1}$};

    \draw[attack] (e_T) -- node[right, font=\tiny] {$\text{}$} (t_tau); 
    \draw[mutual] (t_0) -- node[above, font=\tiny] {$\text{}$} (t_1); 

    \draw[counter, bend left=35] (t_0) to node[left, font=\tiny] {} (e_T);
    \draw[counter, bend right=35] (t_1) to node[right, font=\tiny] {$t_{\tau\cdot 1} \to e^T_\tau$} (e_T);

    \node[box, fit=(e_T) (t_tau) (t_0) (t_1), label=above:\textbf{$T$-Tree (Challenger)}] (t_tree) {};

    \node[dom, below=2.8cm of $(p_1)!0.5!(t_0)$] (U) {$U$};

    
    \draw[eval] (p_sigma) -- node[above, font=\footnotesize, sloped] {$R(e,\sigma,\tau) = \text{True}$} (t_tau);

    \draw[attack] (t_tau.south) .. controls +(down:1.0) and +(right:1.2) .. (U.east) 
        node[pos=0.4, right, font=\tiny] {$t_\tau \to U$};

    \draw[attack, bend left=20] (U) to node[left, pos=0.6, font=\tiny] {} (p_1);
    \draw[attack, bend right=20] (U) to node[right, pos=0.6, font=\tiny] {} (t_0);

\end{tikzpicture}
\end{figure}

We show that $e \in \mathcal{E} \iff \AF_e \in \UniPref$.

($\Rightarrow$)
Suppose $e \in \mathcal{E}$. Then there exists a witness path $P^* \in 2^\omega$ such that for every challenger path $T \in 2^\omega$, there is an integer $n_T$ where $R(e, P^* \upharpoonright n_T, T \upharpoonright n_T)$ holds. Let $S_{P^*} = \{p_{P^* \upharpoonright k} \mid k \in \omega\}$. For any path $T$, the attack $p_{P^* \upharpoonright n_T} \to t_{T \upharpoonright n_T}$ defeats argument $t_{T \upharpoonright n_T}$. Consequently, $t_{T \upharpoonright n_T}$ fails to defend against $e^T_{T \upharpoonright n_T - 1}$, which in turn defeats $t_{T \upharpoonright n_T - 1}$. By induction, this collapse propagates upward to the root, destroying all admissible sets along $T$. Since every challenger node $t_\tau$ is either directly defeated by $S_{P^*}$ or falls to a cascading collapse, no $t_\tau$ survives to attack $U$. Thus, $S_{P^*}$ defends $U$, making $U$ admissible. The set $E^* = S_{P^*} \cup \{U\}$ forms a preferred extension. Since $U \in E^*$ defeats all alternative candidate paths in the $P$-tree and $T$-tree, $E^*$ is the \textbf{unique} preferred extension of $\AF_e$.

($\Leftarrow$) 
Suppose $e \notin \mathcal{E}$. Then for every path $P \in 2^\omega$, there exists a counter-path $T_P \in 2^\omega$ such that $R(e, P \upharpoonright n, T_P \upharpoonright n)$ is False for all $n \in \omega$. For any candidate path $P$, the argument set $S_P = \{p_{P \upharpoonright k} \mid k \in \omega\}$ fails to launch any attack against $T_P$. Thus, the nodes $\{t_{T_P \upharpoonright n} \mid n \in \omega\}$ remain undefeated and self-defending against their enforcers. Since these nodes continuously attack $U$ ($t_{T_P \upharpoonright n} \to U$), $U$ can never be defended and cannot belong to any admissible set. Without $U$ to eliminate competing choices, each distinct path $P \in 2^\omega$ paired with its counter-path $T_P$ yields a distinct preferred extension:
\[
E_P = \{p_{P \upharpoonright k} \mid k \in \omega\} \cup \{t_{T_P \upharpoonright n} \mid n \in \omega\}
\]
Because there are uncountably many choices for $P$, $\AF_e$ possesses uncountably many preferred extensions. Hence, $\AF_e \notin \UniPref$.

\medskip
This completes the reduction and proves that $\UniPref$ is $\Sigma_2^1$-hard.
\end{proof}

\begin{theorem}
$\UniPref$ is $\Sigma_2^1$-complete.
\end{theorem}

\begin{proof}
Hardness follows from Theorem~\ref{hd} (reduction from the canonical $\Sigma_2^1$-complete problem). For membership, $\AF \in \UniPref$ if and only if there exists a preferred extension $E_1$ such that all preferred extensions $E_2$ equal $E_1$. Since being a preferred extension is a $\Pi_1^1$ property, as in Theorem~\ref{pu}, the predicate $\exists^1 E_1 (\text{Pref}(E_1) \land \forall^1 E_2 (\text{Pref}(E_2) \to E_1 = E_2))$ contains two alternating second-order quantifiers ($\exists^1 \forall^1$), placing $\UniPref$ in $\Sigma_2^1$.
\end{proof}

\section{Related Work}

The formal study of computability and expressiveness in infinite argumentation was pioneered by Dung ~\cite{dung1995}, who established fundamental properties such as the $\omega$-continuity of the characteristic function in finitary frameworks. Baumann and Spanring ~\cite{Baumann2017ASO}established the theoretical baseline for unrestricted AFs (UAFs), examining whether fundamental semantic properties hold when argument sets are infinite. A key structural challenge in UAFs is the emergence of infinitely ascending Strongly Connected Component chains, which break the classical SCC-recursive schema ~\cite{Baumann2015}. To address these, Andrews and San Mauro~\cite{EPTCS437.13} proposed repairs using compactness techniques and ordinal-indexing. The computational landscape transitions significantly when moving from finite graphs to infinite structures: Firstly, in \text{Unrestricted IAFs}, complexity shifts to the  analytical hierarchy. Decidability for primary semantics (admissible, stable, complete) inhabits $\Sigma_1^1$-completeness or $\Pi_1^1$-completeness ~\cite{andrews2024infinite}. Then, under finitary restrictions, the characteristic operator becomes $\omega$-continuous, causing a complexity collapse into the arithmetical hierarchy~\cite{andrews2025complexity}. Andrews and San Mauro further showed that stable extensions can code arbitrary $\Pi_1^0$ classes, proving that IAFs are computationally universal for tree-search problems ~\cite{andrews2024infinite}. 

\section{Conclusions}
In this paper, we establish a computational classification for computably finitary AFs:
\begin{itemize}
    \item \text{Grounded Semantics:} Credulous and skeptical acceptance coincide and are $\Sigma_1^0$-complete within $\mathrm{Tot}\times\omega$, as established by Andrews and San Mauro~\cite{andrews2025complexity}. Proposition~\ref{prop:grounded-nonempty} shows that non-empty existence is $\Sigma_1^0$-complete within $\mathrm{Tot}$. Existence and uniqueness are trivial on valid indices.
    \item \text{Preferred Semantics:} Some problems like $\Cred_{\pref}$ ($\Pi_1^0$-c) and $\NE_{\pref}$ ($\Sigma_2^0$-c) remain arithmetical. In contrast, others escape the arithmetical hierarchy: universal quantification over infinite extensions pushes $\Skep_{\pref}$ to $\Pi_1^1$, while ruling out competing infinite branches renders $\Uni_{\pref}$ $\Sigma_2^1$-c.
\end{itemize}

Several directions remain open. While we showed $\Skep_{\pref} \in \Pi_1^1$, its exact completeness level within the analytical hierarchy is still undetermined. Finally, it remains to investigate whether similar collapses occur for other semantics.

\bibliographystyle{splncs04}
\bibliography{NCML}

\end{document}